\documentclass[letterpaper, 10 pt, conference]{ieeeconf}  % Comment this line out
\IEEEoverridecommandlockouts                              % This command is only
\usepackage{amsmath}
\usepackage{times}
\usepackage{array}
\usepackage{booktabs}
\usepackage{amsfonts}
\usepackage{xcolor}
\usepackage{diagbox}
\usepackage{tikz}
\usepackage{graphicx}
\usepackage{subcaption}
\usepackage{balance}
\usepackage{makecell}
\usepackage{hyperref}

\newcommand{\conv}{\mathrm{conv}}

\newcommand{\idist}{\beta}
\newcommand{\bnd}{\mathbf{\tau}}
\newcommand{\const}{\mathbf{c}}

\newcommand{\Real}{\mathbb{R}}

\newcommand{\cmdpM}{\mathcal{M}_c}

\newcommand{\Prob}{\mathcal{P}}

\newcommand{\prim}{V}

\newcommand{\X}{\mathcal{S}}
\newcommand{\A}{\mathcal{A}}

\newtheorem{theorem}{Theorem}
\newtheorem{proposition}{Proposition}
\newtheorem{definition}{Definition}
\newtheorem{problem}{Problem}
\newtheorem{remark}{Remark}

\newcommand{\rev}[1]{#1}

\title{\LARGE \bf
Initial Distribution Sensitivity of Constrained Markov Decision Processes}

\author{Alperen Tercan and Necmiye Ozay% <-this % stops a space
\thanks{This work was supported in part by ONR CLEVR-AI MURI (\#N00014-21-1-2431). The authors are with the Electrical Engiinering and Computer Science Department at the University of Michigan, Ann Arbor. Emails: $\{$tercan,necmiye$\}$@umich.edu
}
}

\begin{document}

\maketitle
\thispagestyle{empty}
\pagestyle{empty}

%%%%%%%%%%%%%%%%%%%%%%%%%%%%%%%%%%%%%%%%%%%%%%%%%%%%%%%%%%%%%%%%%%%%%%%%%%%%%%%%
\begin{abstract}%
 Constrained Markov Decision Processes (CMDPs) are notably more complex to solve than standard MDPs due to the absence of universally optimal policies across all initial state distributions. This necessitates re-solving the CMDP whenever the initial distribution changes. In this work, we analyze how the optimal value of CMDPs varies with different initial distributions, deriving bounds on these variations using duality analysis of CMDPs and perturbation analysis in linear programming. Moreover, we show how such bounds can be used to analyze the regret of a given policy due to unknown variations of the initial distribution.

\end{abstract}

%%%%%%%%%%%%%%%%%%%%%%%%%%%%%%%%%%%%%%%%%%%%%%%%%%%%%%%%%%%%%%%%%%%%%%%%%%%%%%%%

\section{Introduction}

\looseness-1 Constrained Markov Decision Processes (CMDPs) are a fundamental framework for modeling sequential decision-making problems where an agent seeks to maximize cumulative rewards while satisfying certain constraints. These processes extend the classical Markov Decision Processes (MDPs) \cite{puterman2014markov} by incorporating constraints on expected costs or resources, making them highly relevant in fields such as operations research, robotics, telecommunications, and finance \cite{altman1999constrained}. In many real-world applications, agents must not only optimize performance but also adhere to limitations like energy consumption, risk levels, or time budgets \cite{feinberg2000constrained}.

A significant challenge in CMDPs is the lack of uniformly optimal policies across all possible initial state distributions. Unlike standard MDPs, where a stationary policy is optimal regardless of the starting state due to the Markov property and the structure of the optimization problem \cite{bertsekas1995dynamic}, the presence of constraints means that a policy optimal for one initial distribution may not be optimal for another. This issue becomes particularly problematic in environments where the initial state distribution is uncertain or varies over time, leading to potential suboptimal performance if the policy is not adapted accordingly \cite{chow2017risk}.

Traditional approaches often involve solving the CMDP for a specific nominal initial distribution, yielding a policy tailored to that scenario \cite{altman1999constrained}. However, when the actual initial distribution deviates from the nominal one, the performance of this policy can degrade, and re-solving the CMDP at deployment time for the new distribution can be computationally intensive \cite{feyzabadi2017planning}. This is especially impractical in dynamic environments where quick decision-making is crucial. \rev{While the continuity and stability of CMDPs under these deviations have been studied \cite{altman1991sensitivity,altman2002stability}, there is currently a lack of numerical methods for quantifying or analyzing the resulting performance impact.}

\looseness-1 One could instead consider searching for robust policies for uncertain CMDPs, where the uncertainty can be in the initial condition or transition probabilities. However, as opposed to MDPs, robust CMDPs are non-trivial to solve, leading to bilinear optimization problems \cite{varagapriya2024constrained} or heuristic algorithms with local convergence guarantees \cite{zhang2024distributionally, russel2020robust, wang2022robust}. We note that when uncertainty is limited to the initial distribution, an optimal policy robust for all possible initial distributions can be found within the class of non-Markovian policies—specifically, by conditioning the policy on the initial state. However, this requires solving a separate problem for each initial state, which becomes intractable in large state spaces.

%{\color{red}When uncertainty is limited to the initial distribution, an optimal policy for all possible initial distributions can be found within the class of non-Markovian policies—specifically, by conditioning the policy on the initial state. However, this requires solving a separate problem for each initial state, which becomes intractable in large state spaces. In contrast, when restricting to Markovian policies, robust CMDPs—unlike robust MDPs—are non-trivial to solve. The resulting formulations often lead to bilinear optimization problems \cite{varagapriya2024constrained}, or depend on heuristic algorithms with only local convergence guarantees \cite{zhang2024distributionally, russel2020robust, wang2022robust}.}

In this paper, we address this problem by presenting three different methods to derive upper bounds on the expected value function under any arbitrary initial distribution, using the solution obtained from a nominal distribution. Our approach can enable practitioners to estimate the potential performance loss without the need to recompute the optimal policy for each new initial distribution. These upper bounds act as a diagnostic tool to determine whether the existing policy remains sufficiently effective or if re-optimization is warranted under the new conditions. We further show how these bounds can be used to compute sets of initial conditions for which the nominal policy provides a robust solution.

% {\bf Notation:}  

%%% Simplex notation is Delta_ instead of Delta()
\section{Preliminaries}\label{sec:preliminaries}

In this section, we introduce the necessary background for our analysis.

\subsection{Constrained Markov Decision Process (CMDP)}

A CMDP is defined by the tuple \( \cmdpM = (\X, \A, \Prob, r, \const, \bnd, \gamma, \idist) \). Here, \( \X \) denotes a finite set of states, and \( \A \) represents a finite set of actions. The transition probability function \( \Prob: \X \times \A \to \Delta_{\X} \)\footnote{$\Delta_{\X}$ is the probability simplex over $\X$.} gives the probability \( \Prob(s' | s, a) \) of reaching state \( s' \) from state \( s \) after action \( a \). The reward function \( r: \X \times \A \to \mathbb{R} \) defines the immediate reward \( r(s, a) \) received upon taking action \( a \) in state \( s \). The vector-valued constraint function \( \const: \X \times \A \to \mathbb{R}^K \) defines utilities associated with state-action pairs for each of constraints $i=1,\dots,K$. Similarly, \( \bnd \in \mathbb{R}^K \) is the constraint thresholds. The discount factor \( \gamma \in [0, 1) \) adjusts the weight of future rewards, while \( \idist: \in \Delta_{\X} \) is the initial distribution over states, with \( \idist(s) \) as the probability of starting in state \( s \). 

For a given policy \( \pi \) and an initial distribution \( \beta \), the expected discounted cumulative sum for any function \( f: \X \times \A \to \mathbb{R} \) is defined as:
\[
V_f^\pi(\beta) = \underset{\substack{A_t \sim \pi \\ S_0 \sim \beta}}{\mathbb{E}} \left[ \sum_{t=0}^\infty \gamma^t f(S_t, A_t) \right].
\]
This function is linear in $\beta$, since it can be written as $\beta^\top V_f^\pi$ where $i^\text{th}$ entry of $V_f^\pi \in \Real^{|\X|}$ defined as:
\begin{equation}\label{eq:value_vector}
(V_f^\pi)_i = \underset{\substack{A_t \sim \pi}}{\mathbb{E}} \left[ \sum_{t=0}^\infty \gamma^t f(S_t, A_t) \mid S_0 = s_i\right].
\end{equation}
The goal in a CMDP is to find a policy \( \pi: \X \to \A \) that maximizes the expected discounted cumulative reward while satisfying constraints on the cumulative utility. Formally, the optimization problem is:

\begin{equation}\label{eq:cmdp-problem-def}
    \begin{aligned}
    \max_{\pi} & \;\; V_r^\pi(\beta) \\
    \text{subject to} & \;\; V_{c_i}^\pi(\beta) \geq \bnd_i, \quad i \in [K],
    \end{aligned}
\end{equation}
where $[K]$ denotes the set $\{1,\dots,K\}$. 
The optimal value of this problem, as a function of the initial distribution \( \beta \), and constraint threshold vector $\tau$, is denoted by \( \prim^*(\beta, \tau) \). When $\tau$ is clear from the context, we simply write $V^*(\beta)$.
% As a function of the initial distribution \( \beta \) and bound vector $\tau$, this problem is denoted by \( \prim(\beta, \tau) \) and its optimal value is \( \prim^*(\beta, \tau) \).

For a given policy $\pi$, it is useful to have a measure of how much worse it is compared to the optimal policy. \textit{Regret} has been used as the difference between the value of a policy and the optimal policy. In a CMDP setting, we define a generalized notion of regret that also allows quantifying regret for infeasible policies.   

\begin{definition}\label{def:eps-regret}
A policy $\pi$ for the CMDP $(\X, \A, \Prob, r, \const, \bnd, \gamma, \idist)$ is said to incur $(\delta, \epsilon)$-regret if it satisfies the following conditions:
\begin{align*}
    &\tau_i - V_{c_i}^\pi(\beta) \leq \delta_i \qquad \forall i \in [K], \tag{$\delta$-Violation} \label{eq:delta-violation}\\
    &\prim^*(\beta, \tau-\delta) - V_r^\pi(\beta) \leq \epsilon.\tag{$\epsilon$-Regret} \label{eq:epsilon-regret}
\end{align*}
\end{definition}
\vspace{2mm}

We define the minimal regret of a policy $\pi$ as the pair $(\delta^*, \epsilon^*)$ satisfying $(\delta^*, \epsilon^*) \le (\delta,\epsilon)$ for every $(\delta,\epsilon)$ such that $\pi$ incurs $(\delta,\epsilon)$-regret. Here, $(\delta,\epsilon) \le (\delta',\epsilon')$ means that $\delta_i \le \delta'_i$ for all $i$ and $\epsilon \le \epsilon'$. Although this partial order is not total, a minimum always exists: each component of $\delta$ can be minimized independently under \eqref{eq:delta-violation}, and the minimal $\epsilon$ satisfying \eqref{eq:epsilon-regret} is non-decreasing in each component of $\delta$. This follows from the fact that $V^*(\beta,\tau-\delta)$ depends on $\delta$ in a non-decreasing manner—relaxing any constraint does not reduce the optimal value. Moreover, $(\delta^*, \epsilon^*)$ is the unique solution to the conditions in Definition~\ref{def:eps-regret} when the inequalities are replaced with equalities.

% Defining $\pi^*$ to be the solution of the relaxed problem with constraint bounds $\tau - \epsilon_{1:K}$ ensures that infeasible policies cannot improve upon the
This definition allows $\pi$ to violate the constraints up to $\delta$, but effectively ``penalizes" such violations by using $V^*(\beta,\tau-\delta)$ instead of $V^*(\beta, \tau)$ in Equation~\eqref{eq:epsilon-regret} as higher $\delta$ values increase the overall regret. In this paper, we focus on the set of initial distributions that are feasible under $\pi$; hence, we consider only $(0, \epsilon)$-regret policies for ease of exposition. However, we will remark on the extensions for the non-zero $\delta$ case whenever appropriate.

\subsection{Linear Program (LP) formulation of CMDPs}\label{sec:cmdp-lp}

In addition to the policy-space formulation in Equation~\eqref{eq:cmdp-problem-def}, a CMDP can be expressed as an LP in terms of occupation-measures $\rho \in \Real^{|\X||\A|}$ as follows:

\begin{equation}\label{eq:lp-primal}
\begin{aligned}
\max_{\rho \geq 0}& \quad  r^\top \rho \\
\text{ s.t. }& \quad C\rho \geq \tau, \\ 
&\quad \Psi\rho = \beta ;
\end{aligned}\tag{P}
\end{equation}
% where $k^{th}$ row of $C$ corresponds to state-action constraints corresponding to $k^{th}$ constraint. $\Psi$ is 
where:
\begin{align*}
      C &= \left[\begin{array}{ c | c | c | c }
    c^{a_1} & c^{a_2} & \cdots & c^{a_{|\A|}}\\
  \end{array}\right],\\
      \Psi &= \left[\begin{array}{ c | c | c | c }
    I - \gamma T^{a_1} & I - \gamma T^{a_2} & \cdots & I - \gamma T^{a_{|\A|}}\\
  \end{array}\right],
\end{align*}
with $(c^{a_k})_{i,j} = c_i(s_j, a_k)$ and $T^{a_k}_{ij} = \Prob(s_i | s_j, a_k)$.

Once an optimal occupation measure $\rho^*$ for Problem~\eqref{eq:lp-primal} is computed, a policy $\pi$ inducing $\rho^*$ can be obtained by setting $\pi(a|s) = \rho^*(s,a)/\sum_{a' \in \A} \rho^*(s,a')$. For states with $\sum_{a' \in \A} \rho^*(s,a') = 0$, $\pi$ can be chosen arbitrarily as such states would be unreachable. 

We can also write the dual LP corresponding to Problem~\eqref{eq:lp-primal} as:
\begin{equation}\label{eq:lp-dual}
\begin{aligned}
\min_{W, \lambda \geq 0}& \quad  \beta^\top W - \lambda^\top \tau \\
\text{ s.t. }& \quad \Psi^\top W  \geq r + \lambda^\top C; \;
\end{aligned}\tag{D}
\end{equation}
where $W \in \Real^{|\X|}$ and $\lambda \in \Real^K$. We denote an optimal solution corresponding to the problem instance with $\beta$ as $W^*(\beta)$ and $\lambda^*(\beta)$.

\section{Obtaining Bounds for $\prim^*(\beta_1)$}\label{sec:bounds}

In this section, we describe three approaches for obtaining bounds on $\prim^*(\beta_1)$ without explicitly solving an optimization problem for the initial distribution $\beta_1$.

\subsection{LP Dual Feasibility Bound}\label{sec:bnd-dual}

Our first bound uses the fact that the solution of Problem~\eqref{eq:lp-dual}
% feasibility condition of
% on Eq.~\eqref{eq:lp-primal}
at a given nominal initial distribution can be used to bound the optimal performance at different distributions. 

\begin{theorem}\label{thm:lagrangian-bound}
Let $\beta_0$ be the nominal initial distribution the CMDP is solved for; with an optimal dual solution $\lambda^*({\beta_0})$ and $W^*(\beta_0)$. For any initial distribution $\beta_1$, the value $\prim^*(\beta_1)$ satisfies:
\begin{equation}\label{eq:lagrangian-bound}
     \prim^*(\beta_1) \leq \beta_1^\top W^*(\beta_0) - \tau^\top \lambda^*(\beta_0).
\end{equation}
\end{theorem}
\vspace{1em}
\begin{proof}

As the pair $\lambda^*({\beta_0})$ and $W^*(\beta_0)$ is an  optimal solution of Problem~\eqref{eq:lp-dual} for $\beta_0$, it is still feasible when the Problem~\eqref{eq:lp-dual} is solved for $\beta_1$. This is because $\beta$ only appears in the cost function of Problem~\eqref{eq:lp-dual} but does not impact its feasible set.
Hence, $\lambda^*({\beta_0})$ and $W^*(\beta_0)$ are feasible for any initial distribution $\beta_1$.
As Problem~\eqref{eq:lp-dual} is a minimization problem, the value of its objective function at any feasible solution would yield an upper bound on its optimal value, $\prim^*(\beta_1)$. Then, by plugging in $\lambda^*({\beta_0})$ and $W^*(\beta_0)$ into the objective function of Problem~\eqref{eq:lp-dual}, we obtain the desired result.

\end{proof}

\begin{remark}\label{rem:delta-duality}
To analyze $(\delta, \epsilon)$-regret policies, it is necessary to bound $V^*(\beta_1, \tau - \delta)$. The bound derived in this section extends naturally to bound $V^*(\beta_1, \tau - \delta)$. Let $\lambda^*(\beta_0)$ and $W^*(\beta_0)$ be the optimal solution of dual problem of  $\cmdpM = (\X, \A, \Prob, r, \const, \tau, \gamma, \beta_0)$. Then, for any initial distribution $\beta_1$ and violation $\delta$, the value $\prim^*(\beta_1, \tau - \delta)$ satisfies:
\begin{equation}\label{eq:lagrangian-bound-delta}
    \prim^*(\beta_1, \tau - \delta) \leq \beta_1^\top W^*(\beta_0) - (\tau - \delta)^\top \lambda^*(\beta_0).
\end{equation}
\end{remark}

\subsection{Bound from Perturbation of LP}

An alternative bound can be obtained by
treating different initial distributions as perturbations from the nominal initial distribution. Since this perturbation corresponds to a perturbation of the right-hand side vector of equality constraint in Problem~\eqref{eq:lp-primal}, we can obtain a bound based on perturbation analysis of LPs. 
In particular, we adapt the method in~\cite{nayakkankuppam1999conditioning}, summarized with the following theorem.

\begin{theorem}\label{thm:lp-perturbation} (from \cite{nayakkankuppam1999conditioning})
Consider Problem~\eqref{eq:lp-primal}. Let $\rho^*$ and $[W^*;\lambda^*]\in\Real^{|\X|+K}$ to be the optimal solutions of Problem~\eqref{eq:lp-primal} and \eqref{eq:lp-dual}, respectively.
Then, the change in optimal solutions, $\Delta \rho^* =\rho^*_{\text{new}} - \rho^*$ and $[\Delta W^*;\Delta \lambda^*]=[W^*_{\text{new}}-W^*; \lambda^*_{\text{new}} -\lambda^*]$, under the new initial distribution $\Delta \beta + \beta$ satisfies the following for any $\Delta \beta$:
\begin{equation}\label{eq:pert-solution-change}
    \sqrt{\|\Delta \rho^* \|^2 + \|[\Delta W^*;\Delta \lambda^*]\|^2}  \leq \|R^{-1}\| \|\Delta \beta\|,
\end{equation}
where $\|\cdot\|$ is the 2-norm for vectors and the spectral norm for matrices. $R$ is defined as:
$$
R =
\begin{bmatrix}
R_1 & 0 & 0 \\
0 & R_1^\top & 0 \\
0 & R_2^\top & I
\end{bmatrix},
$$
where $R_1, R_2$ constitute a partition of columns of the matrix $\begin{bmatrix}
    C & -I\\
    \Psi & 0
\end{bmatrix}$ such that $i^\text{th}$ column is in $R_1$ if $i \leq |\X||\A|$ and $\rho^*_i > 0$, or $i > |\X||\A|$ and $c_j^\top \rho > \tau_j$ for $j=i-|\X||\A|$.
\end{theorem}

Using Theorem~\ref{thm:lp-perturbation}, we can prove the following value bound:
\begin{theorem}
Let $\beta_0$ and $\prim^*(\beta_0)$ be a nominal initial distribution and its optimal value. Then, any initial distribution $\beta_1$ satisfies:
\begin{equation}\label{eq:thm-pert-ineq}
\begin{aligned}
&\prim^*(\beta_1)
\leq \prim^*(\beta_0)\; + \\
&\qquad\|\beta_0 - \beta_1\| \min(
\|[\beta_1;-\tau]\|\|R^{-1}\| + \|W_0^*\|,\|R^{-1}\|\|r\|
% \qquad\qquad\qquad\qquad\qquad\qquad\quad
 % &\qquad\qquad\qquad\qquad\qquad\;\|[\beta_1;-\tau]\|\|R^{-1}\| + \|W_0^*\|
).
\end{aligned}
\end{equation}
% where $b_0 = [\tau; \beta_0]$.
\end{theorem}
\begin{proof}
    We start by noting that Equation~\eqref{eq:pert-solution-change} gives us the upper bound on the change in the solution but not value. To obtain the bound on value change, we start by setting $\beta_1 = \beta_0 + \Delta \beta$ and $\rho_1^* = \rho^*_0 + \Delta \rho$. Then, we have 
\begin{equation*}
\begin{aligned}
|\prim^*(\beta_1) - \prim^*(\beta_0)|
    = & |r^\top \rho_1^* - r^\top \rho_0^*| \\
%=&|r^\top(\rho_1^* - \rho_0^*)| \\
%=& |r^\top \Delta \rho| \\
\overset{(a)}{\leq}& \|r\|\|\Delta \rho\|\\
% && (\text{Cauchy-Schwarz})\\
\overset{(b)}{\leq}& \|r\| \|R^{-1}\| \|\Delta \beta\|\\
% && (\text{Equation~\eqref{eq:pert-solution-change}}) \\
\leq& \|r\| \|R^{-1}\| \|\beta_1 - \beta_0\|,
% && (\text{Definition of $\Delta \beta$})
\end{aligned}
\end{equation*}
where inequalities (a) and (b) follow from Cauchy-Schwarz and Equation~\eqref{eq:pert-solution-change}, respectively. Writing the value of the LP using its dual gives us another bound:
\begin{equation*}
\begin{aligned}
    |\prim^*(\beta_1) - &\prim^*(\beta_0)|\\
    =& |(\beta_1^\top W_1^* - \tau^\top\lambda_1^*) - (\beta_0^\top W_0^* - \tau^\top\lambda_0^*)| \\
    % =& |[\beta_1; -\tau]^\top [W_1^*;\lambda_1^*] - [\beta_0;-\tau]^\top[W_0^*;\lambda_0^*]| \\
    =& |\beta_1^\top (W_0^* +\Delta W^*)- \beta_0^\top W_0^* - \tau^\top \Delta\lambda^*| \\
    =& |(\beta_1^\top \Delta W^* - \tau^\top \Delta\lambda^*) + \Delta \beta^\top W_0^*| \\
    % =& |(b_0 + \Delta b)^\top (\eta_0^* + \Delta \eta) - b_0^\top \eta_0^*|\\
    % =& |\Delta b^\top \eta_0^* + (b_0 + \Delta b)^\top \Delta \eta | \\
    \overset{(a)}{\leq}& \|[\beta_1;-\tau]\|\|[\Delta W^*;\Delta \lambda^*]\| + \|\Delta \beta\|\|W_0^*\|\\
    % && (\text{Cauchy-Schwarz})\\
    \overset{(b)}{\leq}& \|[\beta_1;-\tau]\|\|R^{-1}\| \|\Delta \beta\| + \|\Delta \beta\|\|W_0^*\| \\
    % &&(\text{Equation~\eqref{eq:pert-solution-change}}) \\
    =& (\|[\beta_1;-\tau]\|\|R^{-1}\| + \|W_0^*\|) \|\beta_1 - \beta_0\|. %&& (\text{Definition of $b$})
\end{aligned}
\end{equation*}
Here inequalities (a) and (b) again follow from Cauchy-Schwarz and Equation~\eqref{eq:pert-solution-change}, respectively.
To get a tighter upper bound, we compute both of these bounds and take the minimum. This gives us the upper bound in Equation~\eqref{eq:thm-pert-ineq}.
\end{proof}

Since the bounds we proved are on $|\prim^*(\beta_1) - \prim^*(\beta_0)|$, a lower bound can also be trivially obtained:
\begin{equation}\label{eq:pert-lower-bound}
\begin{aligned}
&\prim^*(\beta_1) \geq \prim^*(\beta_0) \; - \\
&\qquad \|\beta_0 - \beta_1\| \min(
\|[\beta_1;-\tau]\|\|R^{-1}\| + \|W_0^*\|, \|R^{-1}\|\|r\|).
 % &\qquad\qquad\qquad\qquad\qquad\;\|[\beta_1;-\tau]\|\|R^{-1}\| + \|W_0^*\|
\end{aligned}
\end{equation}

\begin{remark}\label{rem:delta-perturb}
    Theorem~\ref{thm:lp-perturbation} can be generalized to obtain following upper bound to $V^*(\beta_1, \tau-\delta)$ for any $\beta_1$ initial distribution and $\delta$ violation:
    \begin{equation}
\begin{aligned}
  &V^*(\beta_1, \tau-\delta)\leq \prim^*(\beta_0,\tau) \; + \\
  &\quad\qquad \phi \min(
 \|[\beta_1;-\tau]\|\|R^{-1}\| + \|W_0^*\|, \|R^{-1}\|\|r\|),
\end{aligned}
\end{equation}
where $\phi = \sqrt{\|\delta\|^2 + \|\Delta \beta\|^2}$ and other quantities are as defined in Theorem~\ref{thm:lp-perturbation}.
\end{remark}
\subsection{Bound from Concavity of LP Values}\label{sec:jensen-bound}

Finally, we propose an alternative bound based on a well-known concavity result for linear program (LP) values over the set of consistent right-hand side vectors. Specifically, focusing on Problem~\eqref{eq:lp-primal}, we define the feasible set of \(\rho\) as a function of \( \beta \):
\[
\mathrm{Feas}_{\rho}(\beta) = \{ \rho \geq 0 \mid C\rho \geq \tau, 
\; \Psi\rho = \beta \}.
\]
We then define the set of consistent \( \beta \) values as
\[
\mathrm{Feas}_\beta = \{ \beta \mid \exists \rho \; \text{s.t.} \; \mathrm{Feas}_{\rho}(\beta) \neq \emptyset \}.
\]
Equivalently, this set can be expressed as
\[
\mathrm{Feas}_\beta = \{ \Psi \rho \mid C\rho \geq \tau, \; \rho \geq 0 \},
\]
which is convex, since it is the image of a convex set under a linear transformation. For any \( \beta \in \mathrm{Feas}_\beta \), define the function
\[
F(\beta) = \max_{{\rho} \in \mathrm{Feas}_{{\rho}}(\beta)} {r}^\top {\rho},
\]
which represents the optimal value of Problem~\eqref{eq:lp-primal} as a function of \( \beta \). A direct corollary of Theorem~5.1 in \cite{bertsimas1997introduction} is:
\begin{theorem}\label{thm:concave}
The function \( F(\beta) \) is concave on the domain \( \mathrm{Feas}_\beta \).
\end{theorem}

% Note that, from the definition of $b$ in Section~\ref{sec:cmdp-lp}, we can conclude that the optimal value is a concave function of $\beta$ for a fixed value of $\tau$.
Utilizing this result, we obtain an upper bound for the value of the CMDP under an arbitrary initial distribution $\beta_1$.
\begin{theorem}
% Let $\beta_0$ to be the nominal initial distribution the CMDP is solved for; with the optimal value $\prim^*(\beta_0)$ and optimal dual solution $\lambda^*_{\beta_0}$.
Let $\beta_1$ be an initial distribution over states, yielding CMDP value $\prim^*(\beta_1)$. Moreover, let the cardinality of state space $|\X| = n$ and $\alpha = \max_j \beta_1(j)$ to be largest element of $\beta_1$. Then,
\begin{equation}\label{eq:thm-jensen-ineq}
 \prim^*(\beta_1) \leq \alpha n \prim^*(\upsilon) - \sum_i (\alpha - \beta_1(i)) \prim^*(\delta_i),
\end{equation}
where $\upsilon$ is the uniform distribution over state space and $\delta_i(j)$ is non-zero only at $j = i$. 
\end{theorem}

\begin{proof}
Using Theorem~\ref{thm:concave} and definition of concavity, for all $\lambda$ such that $\lambda_i \geq 0$ for all $i$ and $\sum_i \lambda_i = 1$:
\begin{equation}\label{eq:cmdp-jensen}
\prim^*(\left(\sum_{i=0}^n \lambda_i x_i \right)) \geq \sum_{i=0}^n \lambda_i \prim^*({x_i}).
\end{equation}
Pick $\lambda_0 = \frac{1}{n\alpha}$ and $\lambda_i = \frac{\alpha - \beta_1(i)}{n\alpha}$ for $i > 0$. Then, choose $x_0 = \beta_1$ and $x_i = \delta_i$ for $i > 0$. Observe that:
\begin{align*}
    \sum_{i=0}^n \lambda_i &= \frac{1}{n\alpha} + \sum_{i=1}^n \frac{\alpha - \beta_1(i)}{n\alpha}\\
    &= \frac{1}{n\alpha} + \frac{1}{n\alpha}(n\alpha - \sum_{i=1}^n \beta_1(i)) \\
    &= \frac{1}{n\alpha}(1 + n\alpha - 1) = 1.
\end{align*}
Moreover, for all $i$, $\alpha \geq \beta_1(i) \geq 0$,  which implies $\lambda_i \geq 0$. Hence, they yield a convex combination of $x_i$s. 
Finally, observe that:
\begin{align*}
   \sum_{i=0}^n \lambda_i x_i &= \frac{1}{n\alpha}\beta_1 + \sum_i \frac{1}{n} \delta_i - \frac{1}{n\alpha} \sum_i \beta_1(i)  \delta_i \\
   &= \upsilon + \frac{1}{n\alpha}\beta_1 - \frac{1}{n\alpha}\beta_1 = \upsilon.
\end{align*}
Thus, substituting in Equation~\eqref{eq:cmdp-jensen}:

\begin{equation}
\prim^*(\upsilon) \geq \frac{1}{n\alpha}  \prim^*(\beta_1)  +  \sum_{i=0}^n \frac{\alpha - \beta_1(i)}{n\alpha} \prim^*(\delta_i).
\end{equation}
Multiplying both sides with $n\alpha$ and reorganizing terms yields Equation~\eqref{eq:thm-jensen-ineq}.
\end{proof}

Note that, choosing $x_i = \delta_i$, $\lambda_i = \beta_1(i)$ for $n \geq i \geq 1$, and $\lambda_0 = 0$ in Equation~\eqref{eq:cmdp-jensen}, we can also obtain the following lower bound on the value under $\beta_1$:

\begin{equation}
     \prim^*(\beta_1) \geq \sum_i \beta_1(i) \prim^*(\delta_i).
\end{equation}

\subsection{Comparison of Bounds}

In this section, we discuss advantages and drawbacks of each bound, as summarized in Table~\ref{tab:bound_methods}.

\begin{table}[h!]
\centering
\begin{tabular}{@{}l@{}cccc@{}}
\toprule
\textbf{Method} & \textbf{Complexity} & \makecell{\textbf{Empirical} \\ \textbf{Tightness}} & \textbf{Linear} & \textbf{Lower Bound} \\
\midrule
\textbf{Duality} & 1 CMDP & \textcolor{black}{\checkmark}  & Yes & No \\

\textbf{Perturbation}     & $1$ CMDP & \textcolor{gray}{--}  & No & Yes \\
\textbf{Concavity}           & $(|\X|+1)$ CMDPs &\textcolor{black}{$\times$} & No & Yes \\
\bottomrule
\end{tabular}
\caption{Comparison of the three bounding methods. 
The second column indicates how many CMDP instances must be solved to compute each bound. 
The third column shows each method’s overall empirical tightness, i.e., how close the bound is to the actual least upper bound (\checkmark: tightest, $\times$: least tight, and --: moderately tight). 
The fourth column specifies if the bound is linear in the initial distribution~$\beta$. 
Finally, the last column indicates if the method also provides a lower bound.}
\vspace{-18pt}
\label{tab:bound_methods}
\end{table}

Computing the duality-based and perturbation bounds relies on solving the CMDP under the nominal initial distribution. Specifically, the duality-based bound depends on the solution to Problem~\eqref{eq:lp-dual}, while the perturbation bound requires identifying the zero and strictly positive components of the optimal solution to Problem~\eqref{eq:lp-primal}.

In contrast, calculating the concavity bound involves solving the CMDP for $|\X| + 1$ distinct initial distributions. This effectively requires determining the optimal policy for every initial state, rendering the computation both intractable and impractical. While it may be possible to use over-approximations of the optimal values to simplify the process, such an approach often leads to trivial bounds, as the concavity bound is already the loosest among the three.

The obtained bounds are not easy to compare analytically. Hence, in Section~\ref{sec:experiments}, we compare them numerically by computing the relative gap between the computed upper bound and true value of $\prim^*(\beta_1)$. We found that the duality-based bounds are the smallest upper bounds and concavity yields the largest bounds. Perturbation bounds are within two order of magnitude from duality-based bounds.

% \AT{After this point, I talk about stuff that's going to come in next section. Not sure how to resolve this.}
Another important difference between the proposed approaches is whether they provide also a lower bound to $\prim^*({\beta})$. Both perturbation and concavity-based approaches provide lower bounds as explained in their respective sections. However, duality-based approach does not yield a lower bound. When used in Equation~\eqref{eq:epsilon-regret}, a lower bound provides an outer-approximation of the true $\epsilon$-regret set, in contrast to inner-approximation obtained by the upper bounds.

Moreover, we compare the geometry of the obtained bounds. Ideally, we would want the bound to be linear in $\beta_1$, reducing Equation~\eqref{eq:epsilon-regret} to a halfspace. From the three bounds, only the duality-based bound satisfies this. Perturbation bound involves two-norm of $\beta_1$. Similarly, the concavity bound depends on $\|\beta_1\|_\infty$.

Based on this comparison, we conclude that unless a lower bound is needed, duality-based provides the best trade-off. When a lower bound is needed, perturbation bound provides computational efficiency. The concavity approach is the worst one, both computationally intractable for large state spaces and also cannot provide tight bounds.

{
Finally, we note an alternative use of these bounds. As discussed in Remark~\ref{rem:delta-duality} and Remark~\ref{rem:delta-perturb}, these bounds can be generalized to provide bounds on $V^*(\beta_1, \tau - \delta)$. If we are interested in sensitivity to constraint threshold instead of initial distribution, choosing $\beta_1$ to be the nominal initial distribution gives an upper bound on the optimal value under relaxed or tightened constraints.

}
\section{Applications to robust regret problem}\label{sec:applications}

An important question in the study of the initial distribution sensitivity of CMDPs is determining how robust the optimal policy under the nominal initial distribution is to changes in the initial distribution. In this section, we show how the bounds obtained in Section~\ref{sec:bnd-dual} can be used to solve some problems related to this question.

\begin{problem}\label{problem:find-set}
    Given policy $\pi$ and regret $\epsilon$, find the set of initial distributions $B \subseteq \Delta_{|\X|} $ s.t. $\pi$ has $(0, \epsilon)$ regret.
\end{problem}

Note that for a fixed policy $\pi$,
$V_{c_i}^\pi(\beta)$ can be expressed as $\beta^\top V_{c_i}^\pi$ where $V_{c_i}^\pi$ is defined as in Equation~\eqref{eq:value_vector}. Hence, each of the feasibility constraints is linear in $\beta$ and form a halfspace in $\Real^{|\X|}$. However, $V^*(\beta)$ term in Equation~\eqref{eq:epsilon-regret} is a hard-to-compute, nonlinear function of $\beta$ as finding the optimal value requires solving CMDP under initial distribution $\beta$; rendering Problem~\ref{problem:find-set} intractable.

 To deal with this issue, we propose computing inner approximations of the true $(0, \epsilon)$-regret set by replacing $V^*(\beta)$ with an upper bound. If this bound is linear in $\beta$, this approximation becomes a halfspace constraint as well.

\begin{proposition}\label{prp:find-set-polytope}

The following polytope is an inner-approximation to the set defined in Problem~\ref{problem:find-set}:

\begin{equation}\label{eq:prp-find-set}
    \begin{bmatrix}
        (V_r^\pi - W^*(\beta_0))^\top \\
        (V_{c_1}^\pi)^\top \\
        \vdots \\
        (V_{c_K}^\pi)^\top
    \end{bmatrix}
    \beta \geq 
    \begin{bmatrix}
        -\tau^\top \lambda^*(\beta_0) - \epsilon \\
        \tau_1 \\
        \vdots \\
        \tau_K
    \end{bmatrix},
\vspace{.5em}
\end{equation}
where $W^*(\beta_0)$ and $\lambda^*(\beta_0)$ are the optimal solution of Problem~\eqref{eq:lp-dual} for a nominal distribution $\beta_0$. 
\end{proposition}

While Problem~\ref{problem:find-set} helps finding the set of initial distributions that satisfy some predetermined robustness conditions, it cannot directly tell how robust a predetermined set of initial distributions is. We formulate this problem as below.

\begin{problem}\label{problem:find-eps}
    Given policy $\pi$ and a set $B \subseteq \Delta_{|\X|}$ of initial distributions feasible under $\pi$, find minimum $\epsilon$ such that $\pi$ has $(0,\epsilon)$ regret for all $\beta \in B$. More formally, we solve
\begin{equation}\label{eq:find-eps-true}
\begin{aligned}
    \min \quad& \epsilon \\
\text{s.t.}\quad& V^*(\beta) - V_r^\pi(\beta) \leq \epsilon \quad &&\forall \beta \in B.
% \quad&  V_{c_i}^\pi(\beta) \geq \tau_i \quad \forall i \in \{1, \dots, K\}, &&\forall \beta \in B\\
\end{aligned}
\end{equation}
\end{problem}
\vspace{1em}

As formulated in Equation~\eqref{eq:find-eps-true}, Problem~\ref{problem:find-eps} involves an infinite number of constraints and $V^*(\beta)$ is highly nonlinear and makes the problem hard to solve. Using the upper bound in Equation~\eqref{eq:lagrangian-bound}, we can approximate the constraints with a linear constraint. Substituting $V^*(\beta)$ with its upper bound in the constraints yields:
\begin{equation*}
    \beta^\top W^*(\beta_0) - \tau^\top \lambda^*(\beta_0) - V_r^\pi(\beta) \leq \epsilon \quad \forall \beta \in B,
\end{equation*}
where $\beta_0$ is the nominal distribution used in the upper bound. Reorganizing the terms, we obtain:
\begin{equation}\label{eq:find-eps-linear-const}
    \beta^\top( W^*(\beta_0) - V^\pi_r) \leq \tau^\top \lambda^*(\beta_0) + \epsilon \quad \forall \beta \in B.
\end{equation}
Note that Equation~\eqref{eq:find-eps-linear-const} implies the original constraint in Problem~\ref{problem:find-eps}, representing a stricter condition. Therefore, the problems presented below will yield an upper bound on the true solution.

Next, we propose reformulations that eliminate the infinite constraints under two common uncertainty scenarios.

One such scenario arises when $\beta$ is known to lie within the convex hull of a finite set of candidate distributions $\{\beta_1, \beta_2, \dots, \beta_L\}$; that is, \( B = \conv\{\beta_1, \beta_2, \dots, \beta_L\} \).

\begin{proposition}\label{prp:find-eps-vert}
If $B = \conv{\{\beta_1, \beta_2, \dots, \beta_L\}}$, then the following optimization problem provides an upper bound to true solution of Problem~\ref{problem:find-eps}:
% s
\begin{equation}\label{eq:find-eps-vertex}
\begin{aligned}
    \min \quad& \epsilon \\
\text{s.t.}
\quad & \beta_i^\top( W^*(\beta_0) - V^\pi_r) \leq \tau^\top \lambda^*(\beta_0) + \epsilon
\quad i \in [L].
\end{aligned}
\end{equation}
\end{proposition}
\vspace{1em}
\begin{proof}
This result follows from the fact that the constraint is linear in $\beta$. Therefore, if it holds for all extreme points of the convex hull (i.e., the $\beta_i$), it holds for any convex combination of them.
\end{proof}
One might argue that when $L$ is sufficiently small, Problem~\ref{problem:find-eps} could be solved directly by evaluating the constraint at each $\beta_i$, without resorting to an upper bound. However, due to Theorem~\ref{thm:concave}, the original constraint is concave in $\beta$, and thus verifying it only at the vertices is insufficient to guarantee feasibility over the entire set.

Next, we consider the case where representing $B$ via its vertices is not preferable. For example, Wasserstein-1 balls and box constraints often have too many vertices, making enumeration impractical. In such cases, a reformulation that avoids vertex enumeration is preferred. To achieve this, we leverage the results from \cite{ben2015deriving}, which provide equivalent, tractable reformulations for the robust counterparts of many common nonlinear constraints.

\begin{proposition}\label{prp:find-eps-H}
The problem below finds an upper bound to true solution of Problem~\ref{problem:find-eps} when $B = \{\beta \mid H\beta \leq h\} \subseteq \Delta_{\X}$:
\begin{equation}
    \begin{aligned}
        \min_{z,\epsilon} \quad& \epsilon \\
\text{s.t.} \quad 
        &h^\top z - \tau^\top\lambda^*(\beta_0) \leq \epsilon, \\
        &H^\top z = W^*(\beta_0) - V_r^\pi, \\
        &z \geq 0.
    \end{aligned}
\end{equation}
\end{proposition}
\vspace{1em}
\begin{proof}
Since Equation~\eqref{eq:find-eps-linear-const} is a linear constraint with polytopic uncertainty, the result follows directly from \cite{ben2015deriving}.
\end{proof}

\looseness-1Note that while we focus on polytopic $B$ sets here, results from \cite{ben2015deriving} allow generalizing Proposition~\ref{prp:find-eps-H} to many common uncertainty sets.

Moreover, the analysis above focuses on $(0,\epsilon)$-regret, but it naturally extends to the more general notion of $(\delta,\epsilon)$-regret, where the initial distribution $\pi$ may violate the constraints by at most $\delta$. In this setting, our bounds continue to hold with appropriate adjustments, as indicated in Remark~\ref{rem:delta-duality}. For example, Proposition~\ref{prp:find-set-polytope} generalizes to the following polytope, which provides an inner approximation to the set $B \subseteq \Delta_{\X}$ of initial distributions for which $\pi$ incurs $(\delta,\epsilon)$-regret:
\begin{equation}%\label{eq:prp-find-set}
    \begin{bmatrix}
        (V_r^\pi - W^*(\beta_0))^\top \\
        (V_{c_1}^\pi)^\top \\
        \vdots \\
        (V_{c_K}^\pi)^\top
    \end{bmatrix}
    \beta \geq 
    \begin{bmatrix}
        -(\tau - \delta)^\top \lambda^*(\beta_0) - \epsilon \\
        \tau_1 - \delta \\
        \vdots \\
        \tau_K - \delta
    \end{bmatrix}.
\end{equation}

\section{Experiments}\label{sec:experiments}

In this section, we validate our analysis through illustrative examples.
% \footnote{The code for these experiments is available at \url{https://github.com/alperentercan/CMDP-Sensitivity-Analysis}.}

\subsection{Numerical Comparison of Bounds on Random CMDPs}\label{sec:ex:num-random}

Our first set of experiments evaluates the numerical tightness of the proposed bounds. To this end, we randomly generate 100 CMDPs, each with \( \X = 100 \) states, \( \A = 3 \) actions, and \( K = 2 \) constraints. For each CMDP, the transition probabilities, reward functions, and constraint functions are uniformly sampled, with half of the transitions pruned to avoid fully connected CMDPs.

Moreover, we randomly sample 150 initial distributions \( \beta_1 \) divided into 3 bins according to their total variation (TV) distance from the uniform nominal initial distribution \( \beta_0 \). For each $\beta_1$-CMDP pair, we compute the true value and the upper bounds based on $\beta_0$. The tightness of the bounds is quantified using \textit{relative looseness}, defined as the gap between the true value and the upper bound as percentage of the true value.

\looseness-1 To summarize the results, we compute relative looseness percentage across all CMDPs and \( \beta_1 \) distributions for each distance bin. Then, we report the median of each bin for each method. This provides a comprehensive measure of how the bounds perform across diverse scenarios. The results are summarized in Table~\ref{tab:numeric-tightness-random}. As expected, the larger distance between $\beta_0$ and $\beta_1$ leads to larger errors. However, notice that these are percentage errors; hence, it can be said that duality-based bound is very reliable even for furthest $\beta_1$ points.

\begin{table}[h!]
\centering
\begin{tabular}{|c|c|c|c|}
\hline
\diagbox{\textbf{Bounds}}{\textbf{TV Distance}} & \textbf{0.25-0.5} & \textbf{0.5-0.75} & \textbf{0.75-1} \\ \hline
\textbf{Duality}        & 0.014       & 0.048       & 0.297       \\ \hline
\textbf{Perturbation}   & 1.988       & 3.692       & 9.108       \\ \hline
\textbf{Concavity}      & 9.006       & 18.703      & 61.241      \\ \hline
\end{tabular}
\caption{Median relative looseness of the bounds within TV distance interval for randomly generated CMDPs. For each interval, we sort the relative looseness values of the bounds corresponding to all initial distributions and CMDPs; and report the median where a value of $0$ corresponds to the upper bound being equal to the true value.}
\label{tab:numeric-tightness-random}
\end{table}

\vspace{-1mm}

\subsection{Simple Water Pendulum}

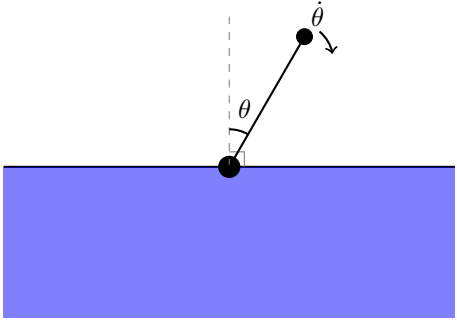
\begin{figure}[t]
    \centering
    % \documentclass{article}
% \usepackage{tikz}
% \usepackage{amsmath}

% \begin{document}

\begin{tikzpicture}
    % Horizon line
    \draw[thick] (-3,0) -- (3,0) node[right] {};% {Water Level};
    
    % Water below the horizon (semi-transparent blue)
    \fill[blue, opacity=0.5] (-3,-2) rectangle (3,0);
    
    % Rotation point (larger than pendulum end)
    \filldraw[black] (0,0) circle (4pt);
    
    % Normal vector (vertical, explicitly marked as perpendicular)
    \draw[gray, dashed] (0,0) -- (0,2);
    \draw[gray] (0.2,0) -- (0.2,0.2) -- (0,0.2);
    
    % Pendulum rod (30 degrees from vertical, above water level)
    % 30 deg from vertical means 60 deg from x-axis
    \draw[thick] (0,0) -- (1,1.732);
    
    % Pendulum mass
    \filldraw[black] (1,1.732) circle (3pt);
    
    % Angle annotation (Theta between the normal [vertical] and the rod)
    % We draw an arc centered at (0,0), from 90 deg (vertical) to 60 deg (the rod)
    \draw[thick] ( {0.5*cos(90)},{0.5*sin(90)} ) arc (90:60:0.5);
    % Label outside the arc
    \node at ( {0.8*cos(75)},{0.8*sin(75)} ) {$\theta$};

    % Angular velocity annotation (clockwise arrow at the end of the pendulum)
    \draw[->, thick] (1.2,1.8) arc (50:10:0.5);
    \node at ({2.35*cos(60)},{2.35*sin(60)}) {$\dot{\theta}$};
\end{tikzpicture}

% \end{document}
\caption{Water pendulum. The objective is to stabilize the pendulum in the upright position, perpendicular to the surface, with both angular position \( \theta \) and angular velocity \( \dot{\theta} \) equal to zero. The blue region represents water.}
\vspace{-8pt}
    \label{fig:water_pendulum}
\end{figure}

While the experiments in the previous section useful for comparing our bounds across diverse set of CMDPs, the random CMDPs are difficult to conceptualize. For this purpose, we use the constrained version of the inverted pendulum task, illustrated in Figure~\ref{fig:water_pendulum}.

The inverted pendulum task involves balancing a pendulum upright by applying a torque input at its pivot point. The state consists of the pendulum angle $\theta \in (-\pi,\pi)$, with upright position assigned $\theta=0$, and $\dot{\theta}$, which takes positive values clockwise. The reward function is negative quadratic in deviation from upright position and angular velocity for all $a \in \A$:
\begin{equation*}
    r([\theta, \dot \theta], a) = -(\theta^2 + \frac{1}{10} \dot \theta ^2),
\end{equation*}
which is between $0$ and $-16.27$.

\looseness-1 In our version, shown in Figure~\ref{fig:water_pendulum}, we assume the pendulum is placed on the surface of a water body such that it is immersed in water for $\theta \geq \pi/2$ or $\theta \leq -\pi/2$. When in water, a buoyant force equal to half of gravity affects the pendulum. We assume two constraint utilities: (1) a -1 utility is incurred every step where pendulum is submerged, (2) an energy budget constraint negatively quadratic of applied torque, $-\frac{1}{4}|a|^2$ where $a$ is the input to system, i.e applied torque. These constraints make the task more challenging and align it with the CMDP framework. For the given mass $m$ and length $l$ of the pendulum, the dynamics of this system is:

\begin{equation}\label{eq:pendulum-dynamics}
\begin{aligned}
% \dot{\theta} &= \dot{\theta}, \\
\ddot{\theta} &=
\begin{cases}
\frac{3}{2} \left( \frac{g}{l} \sin(\theta) \right) + \frac{3}{m l^2} u, & \text{if } |\theta| \geq \frac{\pi}{2}, \\
3 \left( \frac{g}{l} \sin(\theta) \right) + \frac{3}{m l^2} u, & \text{otherwise},
\end{cases}
\end{aligned}
\end{equation}
where $g$ denotes the acceleration due to gravity.

\looseness-1 We discretize the state space by dividing the range of each state dimension into 50 bins. Similarly, the action range is divided into seven. \rev{Transition probabilities are obtained by uniformly sampling 100 points from each bin, simulating them forward under the deterministic continuous dynamics, assigning the end points to bins, and defining probabilities as the fraction of points landing in each target bin.} This results in a two constraint CMDP with $|\X|=2500$ and $|\A| = 7$. We run the experiments described in Section~\ref{sec:ex:num-random}. The results are summarized in Table~\ref{tab:numeric-tightness-pendulum}. We do not use Concavity bound in this setup as large $|\X|$ makes it computationally too expensive.

\begin{table}[h!]
\centering
\begin{tabular}{|c|c|c|c|}
\hline
\diagbox{\textbf{Bounds}}{\textbf{TV Distance}} & \textbf{0.25-0.5} & \textbf{0.5-0.75} & \textbf{0.75-1} \\ \hline
\textbf{Duality}        & 0.000 & 0.001 & 0.004 \\ \hline
\textbf{Perturbation}   & 0.123 & 0.254 & 0.342 \\ \hline
\end{tabular}
\caption{\looseness-1 Median relative looseness of the bounds within each TV distance interval for Water Pendulum. For each interval, we sort the relative looseness values of the bounds corresponding to all initial distributions and report the median.}

\label{tab:numeric-tightness-pendulum}
\end{table}

\subsection{Measuring Policy Robustness by $(0,\epsilon)$-Regret Sets}
Next, we compare the robustness of optimal policies for three initial distributions: $\beta_T$, $\beta_B$, and $\beta_U$. The distribution $\beta_T$ assigns high probability mass when $\theta = 0$, $\beta_B$ when $\theta = \pm \pi$, and $\beta_U$ is the uniform distribution. We propose the size of the $(0,\epsilon)$-regret sets of the optimal policies corresponding to these distributions as a measure of robustness of such policies to changes in the initial distribution.

Using Proposition~\ref{prp:find-set-polytope}, we compute, for each initial distribution, an inner approximation of the set of $\beta$ distributions for which the corresponding optimal policy $\pi$ incurs at most $(0, 0.01)$ regret. As a proxy for the size of these sets, we sample $100$ random distributions from the probability simplex using a Dirichlet distribution and compute the ratio of samples that fall within each polytope. This ratio, or \emph{hit rate}, serves as a surrogate for the size.

We find the hit rates to be $0.26$, $0.01$, and $0.51$ for $\beta_T$, $\beta_B$, and $\beta_U$, respectively. For comparison, if we evaluate the hit rates for the true $(0, 0.01)$-regret sets instead of the inner approximations given by Proposition~\ref{prp:find-set-polytope}, the values are $0.33$, $0.01$, and $0.51$. This close agreement suggests that our inner approximation is tight. Note that hit rates for the true regret sets are computed by re-solving the CMDP for each sample point and computing true regret, which might in general be prohibitive. For instance, for this example, hit rate computation for the true regret sets takes $\sim 16$ minutes while our method takes only $0.03$ seconds. 

Upon inspection, we observe that, in general, the optimal policy for $\beta_T$ becomes infeasible under most other distributions, as it tends to apply excessive torque and thereby violates the second constraint. This behavior arises because, when the pole starts near the top, only minimal torque is required to maintain balance. Since the constraint is enforced in expectation, the policy optimized for $\beta_T$ can afford to use large amounts of torque in other states without violating the constraint under $\beta_T$ itself.

In contrast, the optimal policies for $\beta_B$ and $\beta_U$ are more broadly feasible across different distributions, though they tend to exhibit slightly higher regret. This is because they are more conservative in their torque usage—sometimes overly so—in order to robustly satisfy the constraint across a wider range of scenarios.

\subsection{Finding Minimal Regret over a Set of Distributions}
Based on our earlier observation that the policy induced by $\beta_T$ is not $(0, 0.01)$-regret in a large part of the simplex, we evaluate its transferability to a broader set of initial distributions. Specifically, we consider the convex hull of three vertex distributions sampled from the initial distributions that did not fall into true $(0,0.01)$-regret set in the previous section.
% each supported on a distinct orientation: $\theta \in \{\pm 165^\circ\}$, $\{\pm 130^\circ\}$, and $\{\pm 95^\circ\}$.
We then compute the minimal regret of the policy optimal for $\beta_T$ over this set, as described in Problem~\ref{problem:find-eps}.

\looseness-1 By using Proposition~\ref{prp:find-eps-vert}, we find $0.29$ to be an upper bound to true minimal regret. Next, we evaluate the tightness of this upper bound by comparing it with the solution of Equation~\eqref{eq:find-eps-true} by sampling the set $B$. Note that this approach requires exponential number of samples; hence, we keep the number of vertex distributions low. This approach yields a lower bound to the true regret as we are effectively sampling some constraints from infinite number of constraints. This approach gives $0.24$, so the true minimal regret is within the range $[0.24, 0.29]$.

\balance
\section{Conclusion}
\looseness-1 In this work, we studied how the performance of a CMDP changes under different initial state distributions and proposed three bounding methods—based on dual feasibility, LP perturbation theory, and concavity—to estimate this variation without solving the CMDP repeatedly. Among these, the duality-based bound was consistently the tightest, while the LP perturbation approach offered both upper and lower bounds with moderate accuracy. The concavity-based method was looser and more computationally demanding. We further demonstrated how these bounds enable practical analyses of policy robustness to changing initial distributions. Specifically, we introduced the notion of $(\delta,\epsilon)$-regret to assess robustness, and showed how our bounds efficiently approximate (i) the set of initial distributions where a policy remains near-optimal, and (ii) the policy’s suboptimality over a given set.

Looking ahead, we are interested in extending our analysis to account for uncertainty or changes in the transition dynamics. In this setting, a promising direction is to develop robust and adaptive policies that can accommodate such uncertainties, leveraging the bounds and sensitivity analysis introduced in this work.

\noindent{\bf Acknowledgments:} The authors thank N.S. Aybat (Penn State) and M. Gurbuzbalaban (Rutgers) for providing valuable pointers on LP perturbation analysis.

\bibliographystyle{abbrv}
\bibliography{references}

\end{document}